\documentclass[twoside,11pt]{article}

\usepackage{blindtext}

\usepackage[preprint]{jmlr2e}

\usepackage{amsmath}
\usepackage{tikz}

\usepackage{lastpage}

\usepackage{lastpage}

\firstpageno{1}

\begin{document}

\title{Can Domain Generalization be Guaranteed in Small-Sample Learning? }

\author{\name Hong Zheng \thanks{This work was completed during my Ph.D. period and was the result of a flash of inspiration. I decided to submit it to arXiv and leave its value to the readers. } \email zhinhom@icloud.com} 

\maketitle

\begin{abstract} 
The small-sample learning problem remains a fundamental challenge in machine learning because limited training data lead to unstable model estimation and generalization. Structural Risk Minimization (SRM) has long been regarded as a principled solution under the classical i.i.d. assumption. However, domain generalization (DG) violates this assumption, leaving the theoretical role of SRM in DG largely unexplored. To bridge this gap, we establish the first theoretical guarantees for SRM in DG under mild assumptions. Specifically, based on the concept of stability, we derive learning consistency and generalization error bounds and prove that these bounds become tight when the hypotheses satisfy the stability condition. Building upon this, under a specific hypothesis space assumption, we establish stability, learning, and generalization bounds for SRM. We further discuss the applicability of these bounds to deep learning. This work establishes theoretical foundations for SRM under distribution shifts and sheds light on the design of robust DG algorithms in small-sample scenarios.  
\end{abstract}


\section{Introduction}
\label{sec1} 

Domain generalization (DG) aims to train machine learning (ML) models using several available data domains, ensuring that they perform well on related but inaccessible (unseen) domains~\citep{b20}. This is meaningful because it uses finite sets to develop methods applicable to infinite sets, which can be regarded as a foundation for developing artificial general intelligence. Recently, several studies have explored invariance~\citep{b2, b15} and causality~\citep{b17, b22} from input data pairs, assuming they are essential for generalization, while developing learning approaches based on empirical risk minimization (ERM). By~\citep{b32}, the success of ERM and its variants relies on the assumption of sufficient data, i.e., large-sample learning. However, in real-world applications, such an assumption may not always be satisfied, particularly in scenarios where data collection is expensive. In other words, we have to face the small-sample learning problem (a term introduced by~\citeauthor{b32}) in many learning scenarios. Under such scenarios, model estimation and generalization become unstable~\citep{b5}. 

How can this challenge be addressed? Is it through data augmentation, data generation using large models, or Structural Risk Minimization (SRM)? Data augmentation, whether applied online or offline, violates data independence assumptions and cannot increase the amount of independent samples. Although it may improve empirical performance, model estimation remains theoretically unstable. Moreover, its effectiveness relies heavily on practitioners' experience rather than a principled learning paradigm~\citep{b39}. Recent work~\citep{b29} proves that using generated data can lead to model collapse. Specifically, learning from AI-derived texts causes models to forget rare information as outputs become increasingly homogeneous. Moreover, AI-generated images reproduce training-data semantics and appearances, indicating compromised data independence. Then, SRM becomes an applicable method, as it at least preserves the data independence assumption. However, in DG, data are not identically distributed, leaving the theoretical role of SRM unclear. 

In this work, we establish theoretical guarantees (Theorem~\ref{theo4}) for SRM under DG scenarios with mild assumptions, thereby bridging this gap. First, based on the concept of hypothesis stability, we derive learning consistency and generalization error bounds and prove that the tightness of these bounds is determined by the stability of the hypothesis. By assuming a reproducing kernel Hilbert space, we then establish stability, learning, and generalization bounds for SRM. Furthermore, we extend our theoretical results to deep learning by introducing corresponding regularizations for SRM and analyzing their stability, learning consistency, and generalization error. In conclusion, we contribute by theoretically proving that DG can be guaranteed in small-sample learning. Through this, we reveal that, when data samples are limited, a sufficiently large SRM regularization parameter is required to obtain tighter stability bounds for hypotheses, thereby achieving tighter learning and generalization bounds.

\section{Theoretical Results} 
\label{sec3} 

\subsection{Preliminaries}
\label{sec3.1}

We here introduce some notation and related concepts that will be used to derive our methods. Considering ${\cal X} \subset {\mathbb R}$ and ${\cal Y} \subset {\mathbb R}$ being respectively an input and output space, we have a training set 
\begin{align}
D = \{ Z_1 = (X,Y)_1, \ldots ,Z_E = (X,Y)_E\} \nonumber
\end{align} 
of size $E$ in ${\cal Z} = {\cal X} \times {\cal Y}$ drawn from unknown distributions in ${\cal P}$. Assume that each $Z \in D$ contains $N$ data pairs $(x_j, y_j)$, i.e., $(X,Y): = \{ (x_j, y_j) \} _{j = 1}^N$. Under domain generalization (DG), $ {(X,Y)_i}{ \sim ^{i.i.d.}}{P_i}$ for all $ i \in [E]$, and $ {P_i} \ne {P_{i’}}$ for all $i \ne i’ \in [E]$. Thus, the overall samples follow an independent non-identically distributed (i.n.d.) assumption. The training set $D$ is typically regarded as the source set, while the testing set $D_t$ is regarded as the target set. Similarly, ${P_i} \ne {P_t}$ for all $i \in [E]$ and all $t \in [E_t]$, where $E_t$ denotes the number of target domains in $D_t$. Furthermore, no information from $D_t$ is available during training. 

Given a source set $D$ of domain size $E$, for each $i = 1, \ldots, E$, we construct a modified dataset by replacing the $v$-th domain of $D$ as follows: 
\begin{align}
D^v = \{ Z_1, \ldots ,Z_{v - 1}, Z’_v, Z_{v + 1} \ldots ,Z_E \}, \nonumber 
\end{align}
where the replacement $Z'_v$ is assumed to be drawn from ${\cal P}$ and to satisfy the above distribution assumption. Here, $Z’_v$ also denotes a group of $N$ sample pairs, rather than a single sample pair, which differs from the setting in~\citep{b30}. Only $D$ participates in the learning process, whereas $D^v$ is excluded from learning and is used solely for stability validation. 

We consider a hypothesis $h \in {\cal H} : {\cal X} \to {\cal Y}$ as a deterministic mapping for prediction. Then, we assume that all mappings are measurable and all sets are countable, which does not restrict the generality of the results presented hereafter. To measure the accuracy of predictions from $h$, we consider the cost function $ c:{\cal Y} \times {\cal Y} \to {\mathbb R}$, and the loss of $h$ with respect to data pairs $Z=(X, Y)$ is then defined as 
\begin{align}
\ell (h,Z) = \ell (h,(X,Y)) = c(h(X),Y). \nonumber
\end{align}
Correspondingly, given any set $D$, the expected risk of $h$ is defined by 
\begin{align}
R(h, D) = \frac{1}{E}\sum_{i=1}^E{ {\mathbb E}_{Z \sim P_i}[ \ell(h,Z) ] } = \frac{1}{E}\sum_{i=1}^E{ {\mathbb E}_{Z_i}[ \ell(h,Z_i) ] }, \nonumber 
\end{align} 
and the empirical risk of $h$ is defined by  
\begin{align}
\hat R(h, D) = \frac{1}{EN}\sum_{i=1}^E \sum_{j=1}^N{ \ell(h, (x_j, y_j)_i) }.  \nonumber 
\end{align} 
Given the set $D^v$, we consider only the empirical risk $\hat R(h, D^v)$. In the following, we will also use the simplified notations $\hat R(h) \equiv \hat R(h,D)$, $R(h) \equiv R(h,D)$, and $\hat R’(h) \equiv \hat R(h, D^v)$ when the dataset is clear from the context. 

Note that $R(h)$ is typically defined as $R(h) = {\mathbb E}_{P \in {\cal P}_{E}}{ {\mathbb E}_{Z \in P} [\ell(h, Z)] }$, where ${\cal P}_{E}$ denotes the set of source distributions. Since ${\cal P}_{E}$ contains $E$ observed domains in DG, the outer expectation reduces to the empirical average over domains, i.e.,${\mathbb E}_{P \in {\cal P}_{E}}[\cdot]=E^{-1}\sum_{i=1}^{E}[\cdot]$.

\subsection{Stability and Bounds}
\label{sec3.2} 
We begin by establishing the basic definition of hypothesis stability under DG scenarios. This is important for achieving our goal: we aim to obtain bounds on the learning consistency and generalization error of hypotheses learned by SRM, and we want these bounds to be tight when the hypotheses satisfy the stability condition. 
\begin{definition}[Error Stability] 
\label{defi1}
Any hypothesis $h \in {\cal H}$ has error stability $\beta$ with respect to the loss function $\ell$ if the following holds, $\forall D \in {\cal Z}^E$, $\forall v \in [E]$, 
\begin{align}
\frac{1}{EN}\left| \sum_{Z \in D}\ell(h, Z) -  \sum_{Z' \in D^v}\ell(h, Z')\right| \le \beta, \nonumber 
\end{align}
which can also be written 
\begin{align}
\label{defi1eq1}
\left| {\hat R(h) - \hat R’(h)} \right| \le \beta. 
\end{align}
\end{definition}
\noindent\textbf{Remarks}. (a) Strictly, we should write ${\cal Z}^{E \times N}$ instead of ${\cal Z}^{E}$. For simplicity, we use ${\cal Z}^{E}$ hereafter. (b) The above stability measures the absolute difference between the empirical risks of the same hypothesis $h$ on two datasets, $D$ and $D^v$. This notion differs from the stability definitions in~\citep{b30}, where stability is typically defined in terms of the expected risks of different hypotheses. 

Clearly, a small $\beta$ indicates that the hypothesis exhibits good stability. Based on this definition, we establish the following exponential decay bound, which is derived from a classical concentration inequality. 
\begin{theorem}
\label{theo1}
For any measurable function $R:{\cal H} \times {\cal Z}^E \to {\mathbb R}$, any hypothesis $h \in {\cal H}$ that satisfies the stability in Definition~\ref{defi1}, i.e., $h$ has $\beta$-stability, and any $\varepsilon > 0$, the following inequality holds, $\forall D \in {{\cal Z}^E}$,  
\begin{align}
\label{them1eq1}
P_D\left[ \hat R(h) - R(h) \ge \varepsilon \right] \le \exp \left\{  \frac{- 2EN \varepsilon^2}{\beta^2}  \right\}. 
\end{align}
\end{theorem}
\begin{proof} 
Consider $\hat R(h, D)$ with respect to any hypothesis $h$. Since $h$ has $\beta$ error stability, the following holds, $\forall D \in {\cal Z}^E$, $\forall v \in [E]$, 
\begin{align}
\left| \hat R(h, D) - \hat R(h, D^v) \right| \le \beta. \nonumber 
\end{align} 
This satisfies the condition of Lemma~\ref{applem1}, i.e., 
\begin{align}
\sup_{D \in {\cal Z}^E, Z’_v \in {\cal Z}} \left| \hat R(h, D) - \hat R(h, D^v) \right| \le \beta, \;  \forall v \in [E]. \nonumber 
\end{align}
This implies that $c_v = \beta, \; \forall v \in [E]$, i.e., 
\begin{align}
\sum_{v} \sum_{i} c_v^2 = EN\beta^2, \nonumber 
\end{align}
and $\varepsilon:=EN\varepsilon$, completing the proof. 
\end{proof}
\noindent\textbf{Remarks}. (a) With probability at least $1-\delta$, it holds that $\epsilon_S(h) \le \beta \sqrt{ \ln(1/\delta) / 2EN }$, where $\epsilon_S(h):=R(h) - \hat R(h)$. (b) As expected, a more restrictive stability criterion yields a correspondingly tighter bound. This bound is a McDiarmid-type bound, which is more flexible and capable of estimating nonlinear statistics~\citep{b33}. When no further information is available, the convergence rate of this bound cannot be improved. Therefore, there is no need to derive a lower bound to establish its optimality. (c) This bound establishes the learning consistency between the expected risk and the empirical risk. 

Based on the above learning consistency bound, we observe that stability plays a critical role in learning performance, as the tightness of the bound is closely related to the stability of hypotheses. Namely, a smaller $\beta$ leads to a tighter bound on $\epsilon_S(h)$. We see that $\sqrt{E}$ is scaled by $\beta$, which implies that for a fixed $E$, the magnitude of $\beta$ dominates the tightness, highlighting the necessity of achieving a small $\beta$. 

We then turn our attention to generalization and introduce the following expected risk of $h$ on the target domains, 
\begin{align}
R_T(h) \equiv R(h, D_t) = \frac{1}{E_t} \sum_{t=1}^{E_t}{ {\mathbb E}_{Z \sim P_t}[ \ell(h,Z) ] }.  \nonumber 
\end{align}
This target risk $R_T(h)$ is also defined as average expected loss of hypothesis $h$ over $E_t$ target domains. Accordingly, we establish the following generalization bound. 
\begin{theorem}[Generalization Bound] 
\label{theo2}
Let $R_T(h)$ denote the above definition of the expected risk of any $h \in {\cal H}$ on the target domains, where $h$ has $\beta$-stability. Assume that $0 \le \ell \le M$. Then, the following inequality holds with probability at least $1-\delta$, 
\begin{align}
\label{theo2eq1}
\epsilon_T(h) \le \beta\sqrt{\frac{\ln(1/\delta)}{2EN}} + C\rho, 
\end{align}
where, $\epsilon_T(h):=R_T(h) - \hat R(h)$, $C = (2\sqrt{2}M) / EE_t$, and $\rho = \sum_t\sum_i d_{JS}(P_{t}, P_{i})$. Here, $d_{JS}(P,Q) = \sqrt{JS(P||Q)}$, where $JS(\cdot||\cdot)$ denotes the Jensen-Shannon (JS) divergence. 
\end{theorem}
\begin{proof}
According to Lemma~\ref{applem2}, for any $P_i$ and $Q_j$, we have 
\begin{align}
\left| {\mathbb E}_{P_i}[\ell] - {\mathbb E}_{Q_j}[\ell] \right| \le 2\sqrt{2}Md_{JS}(P_i,Q_j). \nonumber 
\end{align}
Then, extending this inequality to all source and target domains, we obtain 
\begin{align}
\left| \frac{1}{E_t}\sum_{t}^{E_t}{\mathbb E}_{P_t}[\ell] - \frac{1}{E}\sum_{i}^{E}{\mathbb E}_{P_i}[\ell] \right| \le C \sum_t^{E_t}{ \sum_i^E{ d_{JS}(P_t,P_i) } },\nonumber 
\end{align}
where $C = (2\sqrt{2}M) / EE_t$. Let $\rho = \sum_t\sum_i d_{JS}(P_{t}, P_{i})$. By the inequality $|a| - |b| \le |a-b|$, we have 
\begin{align}
R_T(h) - R(h) \le C\rho. \nonumber 
\end{align}
Then, the proof is completed by applying Theorem~\ref{theo1}. 
\end{proof}
\noindent\textbf{Remarks}. (a) The tightness of this generalization bound is still linked to the stability of $h$. Specifically, given $\rho$, a small $\beta$ lead to a tighter generalization bound. (b) This term $\rho$ captures a central challenge in DG, namely that the generalization bound is affected by the distribution discrepancy between the source and target domains. Unfortunately, there is nothing we can do about this term, as access to $D_t$ is unavailable during learning. (c) $d_{JS}(\cdot, \cdot)$ is a metric that satisfies symmetry and the triangle inequality. When the JS divergence is defined using the natural logarithm, $0 \le d_{JS}(\cdot, \cdot) \le \sqrt{\ln 2}$. 

Similarly, the above generalization bound demonstrates the critical role of $\beta$ in generalization. This is IMPORTANT because the stability of $h$ determines both learning and generalization. Once $\beta$ is obtained, we derive the corresponding learning and generalization bounds. Moreover, since $h$ is arbitrary, the hypothesis obtained by SRM also satisfies the above bounds on $\epsilon_S(h)$ and $\epsilon_T(h)$. We then conduct further investigations. 

The following two definitions are introduced as useful tools for our derivation: one is the $\sigma$-admissibility of $\ell$, and the other is optimal stability ($\beta^*$). 
\begin{definition}
\label{defi2}
A loss function $\ell$ defined on ${\cal H} \times {\cal Y}$ is $\sigma$-admissible with respect to ${\cal H}$ is the associated cost function $c$ is convex and the following condition holds, $\forall h,h' \in {\cal H}$, $\forall (X,Y) \in {\cal Z}$, 
\begin{align}
\left| \ell (h,(X,Y)) - \ell (h',(X,Y)) \right| \le \sigma \left| h(X) - h'(X) \right|. \nonumber
\end{align}
\end{definition}
\noindent\textbf{Remarks}. The $\sigma$-admissibility essentially implies that $\ell$ is $\sigma$-Lipschitz continuous. This establishes a relationship between the loss difference of two hypotheses and the distance between the corresponding hypotheses. 

\begin{definition}[$\beta^*$-stability]
Considering the functional $f(h) = | \hat R(h) - \hat R’(h) |$, let $h^* \in \arg\min_{h \in {\cal H}}{f(h)}$ and $h^*$ is non-zero. We say that $h^*$ achieves the minimum error stability $\beta^*$, i.e., 
\begin{align}
\left| \hat R(h^*) - \hat R’(h^*) \right| \le \beta^*. \nonumber 
\end{align}
\end{definition}
\noindent\textbf{Remarks}. (a) Note that $h^*$ is the minimizer of the functional $f(h)$, rather than of $\hat R(h)$ or $\hat R’(h)$, and that $h^*$ is not necessarily unique. (b) This non-uniqueness implies the existence of multiple hypotheses that satisfy this stability condition. 

We define $\Delta h = |h(X) - h^*(X)|, \; \forall h \in {\cal H}, \; \forall X \in {\cal X}$. Then, we have 
\begin{align}
\label{eq1}
\left| \hat R(h) - \hat R'(h) \right| \le \sigma(\Delta h_D + \Delta h_{D^v}) + \beta^*. 
\end{align}
Here, $\Delta h_D$ and $\Delta h_{D^v}$ are defined respectively as $|h(X) - h^*(X)|_D$ and $|h(X) - h^*(X)|_{D^v}$, where the subscripts $D$ and $D^v$ denote the sets over which the values of $X$ are taken. 
\begin{proof}
We omit the hat notation in $\hat R$ for convenience, i.e., $R$ denotes $\hat R$. Since, we have ${h^ * } \in \arg {\min _{h \in {\cal H}}}\left| {R(h) - R'(h)} \right|$.
Then, for any $h \in {\cal H}$, we have 
\begin{align}
\left| {R(h) - R'(h)} \right|& = \left| {R(h) + R({h^ * }) - R({h^ * }) + R'({h^ * }) - R'({h^ * }) + R'(h)} \right| \nonumber \\
&\le \left| {R(h) - R({h^ * })} \right| + \left| {R'(h) - R'({h^ * })} \right| + \left| {R({h^ * }) - R'({h^ * })} \right| \nonumber \\
& \le \left| {R(h) - R({h^ * })} \right| + \left| {R'(h) - R'({h^ * })} \right| + {\beta ^ * } \nonumber \\
& \le \sigma {\left| {h(X) - {h^ * }(X)} \right|_D} + \sigma {\left| {h(X) - {h^ * }(X)} \right|_{{D^v}}} + {\beta ^ * } \nonumber 
\end{align} 
The first inequality holds by the triangle inequality, the second holds under the definition  that $\left| {R({h^ * }) - R'({h^ * })} \right| \le {\beta ^ * }$, and the last holds due to the $\sigma$-admissibility of the loss functions. 
\end{proof} 

The Structural Risk Minimization (SRM) framework for DG is subsequently formulated as $\hat R_r(h) \equiv \hat R_r(h, D, \lambda)$, where 
\begin{align}
\hat R_r(h, D, \lambda) &= \hat R(h, D) + \lambda {\cal N}(h) = \frac{1}{EN}\sum_{i=1}^E{\ell (h,{Z_i})} + \lambda {\cal N}(h). \nonumber
\end{align} 
Here, $\lambda$ is a penalty factor, and ${\cal N}(h)$ represents any regularization for $h$. When using the dataset $D^v$, we obtain $\hat R’_r(h) \equiv \hat R’_r(h,D^v,\lambda )$. 

Before presenting the main results, we establish the following result with respect to the regularizer ${\cal N}(\cdot)$. 
\begin{theorem}
\label{theo3}
Let $\ell$ be $\sigma$-admissible with respect to ${\cal H}$, let ${\cal N}$ be a functional defined on ${\cal H}$, and let $h^*$ has $\beta^*$-stability. Assuming that $\nabla \hat R({h^ * }) \ge 0$, for any $t \in [0, 1]$, any $h \in {\cal H}$, $\hat R_r(h^ *)$, and $\hat R’_r(h^ *)$, we have 
\begin{align} 
\label{lemm1eq1} 
{\cal N}({h^ * }) - {\cal N}({h^ * } + t\Delta h) \le \frac{{\sigma t}}{{2\lambda EN}}(\Delta {h_D} + \Delta {h_{{D^v}}}). 
\end{align}
\end{theorem}
\begin{proof}
We omit the hat notation in $\hat R$ for convenience, i.e., $R$ denotes $\hat R$. Since $ {h^ * } \in \arg {\min _{h \in {\cal H}}}\left| {R(h) - R'(h)} \right|$, for any $h \in {\cal H}$, we have 
\begin{align}
\left| {R({h^ * }) - R'({h^ * })} \right| \le \left| {R({h^ * } + t\Delta h) - R'({h^ * } + t\Delta h)} \right|. \nonumber 
\end{align}
Here, $ ({h^ * } + t\Delta h) (X) \ne {h^ * } (X)$ for any $t > 0$. Then, we have 
\begin{align}
{(R({h^ * }) - R'({h^ * }))^2} \le {(R({h^ * } + t\Delta h) - R'({h^ * } + t\Delta h))^2}. \nonumber 
\end{align}
Let $\Phi (h) = \left| {R(h) - R'(h)} \right|$, we have $0 \in \nabla \Phi ({h^ * })$, which implies that $ \exists g \in \nabla R({h^ * }) \cap \nabla R'({h^ * })$. Note that here $0$ and $g$ belong to the hypothesis class rather than being scalars.  

Now, let $ \phi (t) = R({h^ * } + t\Delta h)$ and $ \phi '(t) = R'({h^ * } + t\Delta h)$. Then, we have 
\begin{align}
{\nabla _t}\phi (t) &= \nabla R({h^ * } + t\Delta h)\Delta h, \nonumber \\
{\nabla _t}\phi '(t) &= \nabla R'({h^ * } + t\Delta h)\Delta h. \nonumber
\end{align}
Consider $\varphi (t) = \phi (t)\phi '(t)$, we have 
\begin{align}
{\nabla _t}\varphi (t) &= \nabla R({h^ * } + t\Delta h)\Delta h\phi '(t) + \nabla R'({h^ * } + t\Delta h)\Delta h\phi (t), \nonumber \\
{\nabla _t}\varphi (0) &= \nabla R({h^ * })\Delta h\phi '(0) + \nabla R'({h^ * })\Delta h\phi (0). \nonumber 
\end{align}
Since $R$ is convex (see Definition~\ref{defi2}), its gradient is a monotone mapping, so we have $ \nabla R({h^ * } + t\Delta h) \ge \nabla R({h^ * })$ for any $t \ge 0$. Moreover, we assume that $\nabla R({h^ * }) \ge 0$. Then, we obtain ${\nabla _t}\varphi (t) \ge {\nabla _t}\varphi (0)$. This implies that 
\begin{align}
R({h^ * } + t\Delta h)R'({h^ * } + t\Delta h) \ge R({h^ * })R'({h^ * }). \nonumber 
\end{align}
Then, we compute
\begin{align}
&{(R({h^ * } + t\Delta h) + R'({h^ * } + t\Delta h))^2} - {(R({h^ * }) + R'({h^ * }))^2} \nonumber \\
=& (R_\Delta ^2 + 2{R_\Delta }{{R'}_\Delta } + {({{R'}_\Delta })^2}) - ({R^2} + 2RR' + {(R')^2}) \nonumber \\
=& (R_\Delta ^2 + 2{R_\Delta }{{R'}_\Delta } + {({{R'}_\Delta })^2} - 2{R_\Delta }{{R'}_\Delta } + 2{R_\Delta }{{R'}_\Delta }) \nonumber \\
&- ({R^2} + 2RR' + {(R')^2} - 2RR' + 2RR') \nonumber \\
=& ({({R_\Delta } - {{R'}_\Delta })^2} + 4{R_\Delta }{{R'}_\Delta }) - ({(R - R')^2} + 4RR') \nonumber \\
=& ({({R_\Delta } - {{R'}_\Delta })^2} - {(R - R')^2}) + 4({R_\Delta }{{R'}_\Delta } - RR'). \nonumber
\end{align}
Based on the previous derivations, we obtain 
\begin{align}
{(R({h^ * }) + R'({h^ * }))^2} &\le {(R({h^ * } + t\Delta h) + R'({h^ * } + t\Delta h))^2} \nonumber \\
R({h^ * }) + R'({h^ * }) &\le R({h^ * } + t\Delta h) + R'({h^ * } + t\Delta h).  \nonumber 
\end{align}
Then, we replace $R(h^*)$ and $R'(h^*)$ with $R_r(h^*)$ and $R'_r(h^*)$, respectively, and obtain 
\begin{align}
{R_r}({h^ * }) + {{R'}_r}({h^ * }) \le {R_r}({h^ * } + t\Delta h) + {{R'}_r}({h^ * } + t\Delta h). \nonumber 
\end{align}
Consequently, we have 
\begin{align}
2\lambda ({\cal N}({h^ * }) - {\cal N}({h^ * } + t\Delta h)) \le (R({h^ * } + t\Delta h) - R({h^ * })) + (R'({h^ * } + t\Delta h) - R'({h^ * })) \nonumber 
\end{align}
Finally, based on the $\sigma$-admissibility of the loss functions and the definition of $R$, we obtain 
\begin{align}
{\cal N}({h^ * }) - {\cal N}({h^ * } + t\Delta h) \le \frac{{\sigma t}}{{2\lambda EN}}(\Delta {h_D} + \Delta {h_{{D^v}}}). \nonumber 
\end{align}
\end{proof}
\noindent\textbf{Remarks}. (a) This result leads to the regularizations associated with the $\sigma$-admissibility of loss functions. (b) Similarly, $h^*$ is not the minimizer of either $\hat R_r(h)$ or $\hat R’_r(h)$. (c) This theorem considers only minimizers $h^*$ satisfying $\nabla \hat R({h^ * }) \ge 0$. 

Next, we impose an assumption on ${\cal H}$ to ensure that ${\cal N}(\cdot)$ is well-defined. Specifically, we consider regularization in a reproducing kernel Hilbert space (RKHS). An RKHS is chosen to ensure efficient kernel computation and a complete, stable function space. Accordingly, we define 
\begin{align}
{\cal N}(h) = \left\| h \right\|_k^2, \nonumber
\end{align}
where $k$ refers to the kernel. Typically, we state that ${\cal H}_k$ denotes the norm in an RKHS, but for simplicity, we use $k$ here. Based on the fundamental property (reproducing) of an RKHS ${\cal H}$, we have 
\begin{align}
\forall h \in {\cal H}, \forall X \in {\cal X},h(X) = \left\langle {h,k(X, \cdot )} \right\rangle . \nonumber
\end{align}
Then, according to Cauchy-Schwarz inequality, we have 
\begin{align}
\label{eq2}
\forall h \in {\cal H},\forall X \in {\cal X}, \left| h(X) \right| \le \left\| h \right\|_k \sqrt {k(X,X)}. 
\end{align} 

With all the previous preparations, we finally present the following result for our main goal. 
\begin{theorem}
\label{theo4}
Assume that ${\cal H}$ is a reproducing kernel Hilbert space with kernel $k$ such that $\forall X \in {\cal X},\sqrt {k(X,X)} \le \kappa \le \infty $. Let $\ell$ be $\sigma$-admissible with respect to ${\cal H}$, and assume that $0 \le \ell \le M$. Any hypothesis $h$ achieved by 
\begin{align}
\label{theo4eq1}
\mathop{\min}_{h \in {\cal H}} \hat R(h) + \lambda\left\|h\right\|_k^2 , 
\end{align}
has error stability $\beta$ respect to $\ell$ with 
\begin{align}
\beta \le \frac{2 \sigma^2 \kappa^2}{\lambda EN}+\beta^* . \nonumber
\end{align} 
Furthermore, with probability al least $1-\delta$, the learning consistency bound holds: 
\begin{align}
\epsilon_S(h) \le \left( \frac{2\sigma^2\kappa^2}{\lambda E N} +\beta^* \right)\sqrt{\frac{\ln(1/\delta)}{2EN}} , \nonumber 
\end{align}
and the generalization bound is given by 
\begin{align}
\epsilon_T(h) \le \left( \frac{2\sigma^2\kappa^2}{\lambda E N} +\beta^* \right)\sqrt{\frac{\ln(1/\delta)}{2EN}} + C\rho. \nonumber 
\end{align}
\end{theorem} 
\begin{proof}
Since ${\cal N}(h) = \left\| h \right\|_k^2$, based on Theorem~\ref{theo3}, we have 
\begin{align}
t\left\| {\Delta h} \right\|_k^2 \le \frac{{\sigma t}}{{2\lambda EN}}(\Delta {h_D} + \Delta {h_{{D^v}}}). \nonumber 
\end{align}
According to the assumption that $\forall X \in {\cal X},\sqrt {k(X,X)} \le \kappa \le \infty$ and Inequality~\eqref{eq2}, we have 
\begin{align}
\Delta {h_D} &\le {\left\| {\Delta h} \right\|_k}\sqrt {k(X,X)}  \le {\left\| {\Delta h} \right\|_k}\kappa, \nonumber \\
\Delta {h_{{D^v}}} &\le {\left\| {\Delta h} \right\|_k}\kappa . \nonumber
\end{align}
Consequently, we obtain 
\begin{align}
\left\| {\Delta h} \right\|_k^2 &\le \frac{\sigma }{{2\lambda EN}}2{\left\| {\Delta h} \right\|_k}\kappa  = \frac{{\sigma {{\left\| {\Delta h} \right\|}_k}\kappa }}{{\lambda EN}}, \nonumber \\
{\left\| {\Delta h} \right\|_k} &\le \frac{{\sigma \kappa }}{{\lambda EN}}\nonumber
\end{align}
Moreover, according to Inequality~\eqref{eq1}, and by taking $\beta  \le \sigma (\Delta {h_D} + \Delta {h_{{D^v}}}) + {\beta ^ * }$, we have 
\begin{align}
\beta  \le 2\sigma {\left\| {\Delta h} \right\|_k}\kappa  + {\beta ^ * }. \nonumber 
\end{align}
By Theorem~\ref{theo1} and Theorem~\ref{theo2}, we complete the proof. 
\end{proof}
\noindent\textbf{Remarks}. (a) Up to an intercept, $\beta$ is inversely proportional to the product of the penalty factor and the total number of data samples. As $\lambda EN \to \infty$, the bound on $\beta$ becomes tight when $\beta^*$ approaches zero. (b) The tightness of the learning and generalization bounds is now linked to the number of data samples within each domain. Namely, the main order of the above bounds is ${\cal O}(1/\sqrt{\lambda E^3N^3})$. (c) Assuming $\beta^* \to 0$, the term involving $\beta^*$ in the above bounds can be regarded as a bias term that is typically negligible. 

Now, we have established stability, learning, and generalization bounds for the hypothesis $h$ learned by SRM in DG. By the theorem, for fixed $\sigma$ and $\kappa$, the tightness of all the above bounds is determined by the parameters $\lambda$, $E$, and $N$. Specifically, a sufficiently large $\lambda EN$ theoretically leads to a smaller $\beta$, which in turn yields better learning consistency and generalization performance.

\subsection{Application} 
\label{sec3.3}
Still, one step remains before applying our theoretical results to practical applications. In most deep learning scenarios, the hypothesis space is assumed to be a Euclidean space. In the previous subsection, we presented the stability, learning, and generalization bounds for SRM under the assumption of an RKHS. For a kernel-induced RKHS, the function norm in the RKHS can be transformed into the Euclidean $L_2(\mu)$ norm in the corresponding feature space through finite-dimensional feature representations or kernel feature mappings. This property provides a theoretical foundation for applying RKHS-based analysis to models defined in Euclidean spaces. The detailed derivation is omitted for brevity. We then consider the following SRM framework:
\begin{align}
\hat R_r(h) = \hat R(h) + \lambda \| h \|_2^2, \nonumber
\end{align}
where the hypothesis $h$ is parameterized by $\theta$, i.e., $h = h_{\theta}$. Correspondingly, the learning objective is 
\begin{align}
\label{eq3}
\mathop{\min}_{h \in {\cal H}} \hat R(h) + \lambda\left\|h\right\|_2^2. 
\end{align}
Note that, for simplicity, we directly use the above Lagrangian formulation as the learning objective instead of the original constrained optimization problem. 

We further assume that, for a finite-dimensional RKHS, the induced kernel is normalized, i.e., 
\begin{align}
k(X,X) \le 1,\quad \forall X \in {\cal X}, \nonumber 
\end{align}
which yields $\kappa=1$. This assumption is equivalent to bounding the norm of the induced feature representation and naturally extends to Euclidean-space models, including deep neural networks, through their learned feature embeddings. 

Accordingly, under the above assumption and settings of Theorem~\ref{theo4}, any hypothesis $h$ achieved by learning objective~\eqref{eq3} has error stability $\beta$ respect to $\ell$ with 
\begin{align}
\beta \le \frac{2 \sigma^2}{\lambda E N}+\beta^* . \nonumber
\end{align} 
Furthermore, with probability al least $1-\delta$, the learning consistency bound holds: 
\begin{align}
\epsilon_S(h) \le \left( \frac{2 \sigma^2}{\lambda E N}+\beta^* \right)\sqrt{\frac{\ln(1/\delta)}{2EN}}, \nonumber 
\end{align}
and the generalization bound is given by 
\begin{align}
\epsilon_T(h) \le \left( \frac{2 \sigma^2}{\lambda E N}+\beta^* \right)\sqrt{\frac{\ln(1/\delta)}{2EN}} + C\rho. \nonumber 
\end{align}

In the above, we only consider the $L_2$-norm. However, in SRM, the $L_1(\mu)$-norm is also commonly considered to promote sparsity. Since there is no direct correspondence between $\| h \|_k^2$ and the $L_1$-norm, we instead consider the following learning objective to promote sparsity, 
\begin{align}
\label{eq4}
\mathop{\min}_{h \in {\cal H}} \hat R(h) +\frac{\lambda}{E} \sum_{i =1}^E \left\|h_i\right\|_2^2. 
\end{align}
The above implies that ${\cal N}(h) = \sum {\| h_i\|}$, and we still obtain the stability bound  
\begin{align}
\beta \le \frac{2 \sigma^2}{\lambda E N}+\beta^* , \nonumber
\end{align}
which yields the same learning and generalization bounds as those for the learning objective~\eqref{eq3}. Here, $h_i$ denotes the parameters of $h$ associated with domain $i$. 

Although the parameters are identical across all domains during the learning process, the summation over domains is introduced for the purpose of sparsity regularization. In essence, $\| h \|_1 = \sum{ | \theta_i | }$, i.e., the sum of the absolute values of all weights in $h$. Similarly, $\sum{\| h_i\|_2^2}$ can also be viewed as a summation over the parameters of $h$, except that each summand is replaced by $\| h_i\|_2^2$, with the summation taken over the $E$ domains. This formulation promotes domain-level sparsity, i.e., inter-domain sparsity, while encouraging smoothness within each domain. Such a consideration is analogous to the ideas underlying LASSO and Group LASSO; further details are omitted here. Additionally, learning objectives~\eqref{eq3} and \eqref{eq4} are strongly convex. Therefore, under gradient-based optimization methods, the optimization converges Q-linearly.

\subsection{Discussion}
\label{sec3.4} 

\noindent\textbf{Stability}. We provide a discussion of Definition~\ref{defi1} and the definitions in~\citep{b30}. The stability in~\citep{b30} measures the difference between the expected losses of two models $h_1$ and $h_2$ trained on $D$ and $D^v$, respectively, thereby evaluating the sensitivity of the learning algorithm to changes in training data. In contrast, our stability measures the difference between the empirical risks of the same hypothesis $h$ evaluated on $D$ and $D^v$. It assesses whether the hypothesis is sensitive to source domain changes. A larger $\beta$ indicates weaker stability. We consider empirical risk instead of expected risk since it is directly computable in practice, making our measure suitable for evaluating model stability in DG.

\noindent\textbf{$\beta^*$-stability}. We may avoid considering $\beta^*$-stability, as well as $h^*$. By Lemma 20 in \citep{b30}, some results regarding ${\cal N}(\cdot)$ associated with the minimizer of SRM can be established, thereby restricting all subsequent theorems to this minimizer. However, by defining $\beta^*$-stability, our theory no longer unnecessarily depends on the restriction to the SRM minimizer. Therefore, in essence, $h^*$ relaxes the restriction that the learning and generalization bounds must be based on the SRM minimizer, resulting in a milder condition for DG scenarios. This serves as a significant novelty and distinguishes our work from theirs.

\noindent\textbf{$\sigma$ and $M$}. One may note that, throughout our main results, we use the notation $\sigma$ to denote the Lipschitz constant and $M$ to denote the uniform bound of the loss function. This is because we do not specify the exact loss formulation or the specific task. Here, we present the following examples for classification and regression tasks, whose derivations rely on Lemma~3 provided in the supplemental material. If ${\cal Y} \in \{-1, 1\}$, and considering the following loss function 
\begin{align}
\ell(f, (x,y))=
\begin{cases} 
1-yf(x), & if \; 1-yf(x)\ge 0 \\
0, & otherwise 
\end{cases}, \nonumber 
\end{align}
we have $\sigma = 1$ and $M=1$. If ${\cal Y} \in \{0, B\}$ and considering the following loss function 
\begin{align}
\ell(f, (x,y)) = (f(x)-y)^2, \nonumber 
\end{align}
we have $\sigma = 2B$ and $M = B$. Other examples can be found in~\citep{b30}.

\noindent\textbf{Previous works}. Recently, under the assumption of sufficient data samples, many studies have developed theoretical analyses of learning under distribution shifts. To the best of our knowledge, our work is the first to investigate small-sample learning in DG from a theoretical perspective. Therefore, a direct comparison with these works is not meaningful, and we only provide some related discussions. One well-known work is presented by~\cite{b34}, which has motivated many subsequent studies on representation learning. Their work is elegant and introduces a challenge term in the upper bound that is similar to our $\rho$. \cite{b36} presented an upper bound for their learning method; however, the connection between the hyperparameters of the learning algorithm and the derived bound was not established. \cite{b38} provides insights into how distribution shift itself affects learning, but their upper bound relies on a strong assumption regarding distribution divergence. Besides upper bounds, a lower bound for DG was presented by~\cite{b5} to explain the phenomenon observed in~\citep{b37}; however, it requires an infinite number of data samples in each domain.

\noindent\textbf{Limitations}. It should be emphasized that our theoretical bounds for SRM are established within the RKHS framework. Consequently, the results naturally extend to learning problems whose hypothesis spaces admit an RKHS formulation, such as deep learning. However, when the hypothesis space does not admit an RKHS formulation, our theoretical results may no longer hold.

\section{Related Works}
\label{sec2} 

We briefly introduce some empirical and theoretical methods in DG. For more details, see the references. 

\noindent\textbf{Empirical Methods}: Since IRM~\citep{b2}, various extensions, including IRM-Games~\citep{b15}, invariant information bottleneck~\citep{b16}, and P-IRM~\citep{b3}, have been proposed, along with studies on invariant representation learning in latent spaces~\citep{b20, b21}. These methods assume that invariant features from observed domains improve generalization to Out-of-distribution domains. However, recent works have challenged this assumption by revealing failure cases of invariance-based methods~\citep{b18} and proposing low-risk alternatives~\citep{b19}. Causality-based approaches~\citep{b17, b22} have also been introduced to address domain shifts but generally require stronger assumptions. In addition, the sparsity of ML models has also gained increasing attention recently, such as in ~\citep{b23, b26, b28}. 

\noindent\textbf{Theoretical Methods}: From the VC bounds~\citep{b4,b32} to the Massart-noisy bounds~\citep{b35}, the mystery of the learning pattern in machine learning has been revealed gradually. These theories provide fundamental inequalities for analyzing learning problems and inspire subsequent studies. \cite{b34} provides an early theoretical study of learning under distribution shift, motivating further investigations in DG. Many efforts have since focused on developing guarantees for well-generalized ML models, e.g., \citep{b5, b38}. Our work contributes to this line of research by addressing learning problems under limited data.

\section{Conclusion}
\label{sec5}
In this study, we present theoretical guarantees for SRM in DG scenarios and provide corresponding applications. Our results imply that, even when data samples are limited, stable, well-performing, and generalizable models can still be learned, i.e., DG can be theoretically guaranteed in small-sample learning. This finding is important for practical applications of ML and deep learning methods, as sufficient training samples cannot be collected in every application field. Moreover, for fields where data collection is highly expensive, our theory provides an alternative solution. However, we cannot theoretically characterize the selection of $\lambda$, which serves as a limitation of our work. We expect to address this issue in future research.

\section*{Ackonwlegement}
I would like to express my sincere gratitude to my supervisors, Prof. T \& Prof. L, for providing me with the time, freedom, and support to pursue my own research interests
and ideas throughout my studies.

\appendix
\section{}
\label{app}

In this appendix, we only provide the lemmas, along with their porous, used to prove our main theorems. 

\subsection{Lemmas}
\label{applem}

\begin{lemma}[McDiarmid’s inequality~\citep{b31}]
\label{applem1}
Let $S:=\{z_1, \ldots, z_n\}$ and $S^i:=\{z_1, \ldots, z_{i-1}, z’_i, z_{1+1}, \ldots, z_n\}$, and let $ F: {\cal Z}^n \to {\mathbb R} $ be any measurable function for which there exist constants $c_i$ $(i=1,\ldots, n)$ such that 
\begin{align}
\mathop {\sup }_{S \in {\cal Z}^n, z'_i \in {\cal Z}} \left| {F(S) - F({S^i})} \right| \le {c_i}  \nonumber 
\end{align} 
then, for any $\varepsilon > 0$, 
\begin{align}
P_{S}\left[ {F(S) - {{\mathbb E}_S}[F(S)] \ge \varepsilon } \right] \le \exp \left\{ {\frac{{ - 2{\varepsilon ^2}}}{{\sum_{i=1}^n {{c_i^2}} }}} \right\}. \nonumber
\end{align}
\end{lemma}
\begin{proof} 
This is a classical concentration inequality, and the proof is omitted here. For more details, please refer to~\citep{b31}. 

For clarity on how this lemma is applied in our case, we provide the following simple example. Consider that $\hat R(h, D): {\cal H} \times {\cal Z}^{E \times N} \to {\mathbb R}$ and $\hat R(h, D^v)$. Here, $D$ and $D^v$ differ from $S$ and $S^i$, respectively. Then, if we have 
\begin{align}
\sup_{D \in {\cal Z}^{E \times N}, Z’_v \in {\cal Z}^N} \left| \hat R(h, D) - \hat R(h, D^v)\right| \le c_v, \nonumber 
\end{align}
it follows that 
\begin{align}
P\left[ \hat R(h, D) - R(h, D) \ge \varepsilon \right] \le \exp\left\{ \frac{-2\varepsilon^2}{\sum_{v=1}^E\sum_{i=1}^N{c_v^2} } \right\}.  \nonumber 
\end{align}
Note the difference in the sum of $c^2$: in the original lemma, $z’_i$ represents a single data pair, whereas in our case, $Z’_v$ contains $N$ data pairs. Therefore, we use a double summation here. 
\end{proof}

\begin{lemma}[Expectation Bound via JS divergence]
\label{applem2}
Let $P$ and $Q$ be two probability distributions. Let $F: {\cal Z} \to {\mathbb R}$ be a measurable function satisfying 
\begin{align}
\left\|F\right\|_{\infty}:=\mathop{\sup}\limits_{Z \in{\cal Z}}{\left| F(Z) \right|}\le M \le \infty. \nonumber 
\end{align}
Then,
\begin{align}
\left| {\mathbb E}_P[F]-{\mathbb E}_Q[F]\right|\le 2M\sqrt{2JS(P||Q)}, \nonumber 
\end{align}
or 
\begin{align}
\left| {\mathbb E}_P[F]-{\mathbb E}_Q[F]\right|\le 2\sqrt{2}M d_{JS}(P,Q). \nonumber 
\end{align}
\end{lemma}
\begin{proof}
The result follows from standard inequalities in information theory. First, by a standard total variation bound, 
\begin{align}
\left| {\mathbb E}_P[F]-{\mathbb E}_Q[F]\right|\le 2\left\|F\right\|_{\infty}d_{TV}(P,Q). \nonumber 
\end{align}
Applying the Pinsker's inequality, we have 
\begin{align}
\left(\frac{1}{2}d_{TV}(P,Q)\right)^2 &\le \frac{1}{2}KL(P||M) \nonumber \\
\left(\frac{1}{2}d_{TV}(P,Q)\right)^2 &\le \frac{1}{2}KL(Q||M). \nonumber 
\end{align}
Here, $M=(P+Q)/2$, $d_{TV}(P,M)=(1/2)d_{TV}(P,Q)$, and $d_{TV}(Q,M)=(1/2)d_{TV}(P,Q)$. Then, by the definition of the JS divergence, i.e., $JS(P;Q)=1/2KL(P;M)+1/2KL(Q;M)$, we obtain 
\begin{align}
2\left(\frac{1}{2}d_{TV}(P,Q)\right)^2 \le \frac{1}{2}JS(P||Q). \nonumber 
\end{align}

Here, we clarify how this lemma can be used to derive some results in our paper. Let $Z=(h(X),Y)$ and consider $\ell$ as $F$. Then, assuming $\|\ell\|_{\infty} \le M$, we have 
\begin{align}
{\mathbb E}_{Z \sim P}[|\ell(Z)|]-{\mathbb E}_{Z \sim Q}[|\ell(Z)|] &\le 2M \sqrt{2JS(P^{(h(X),Y)}||Q^{(h(X),Y)})}, \nonumber \\
&\le 2M \sqrt{2JS(P^{(X,Y)}||Q^{(X,Y)})}, \nonumber
\end{align}
considering $|a| - |b| \le |a - b |$. The last inequality is a consequence of the Data Processing Inequality from information theory. Note that $h$ is the same for both $P$ and $Q$. If this condition does not hold, the last inequality will not hold either. 
\end{proof}

\begin{lemma}
\label{applem3}
Let $h$ be the hypothesis obtained by SRM where $\ell$ is a loss function associated with a convex cost function $c(\cdot, \cdot)$. We denote by $B(\cdot)$ a positive non-decreasing real-valued function such that for all $y \in {\cal Y}$. Then, we have 
\begin{align}
\forall y' \in {\cal Y},c(y,y') \le B(y). \nonumber 
\end{align}
Moreover, $\ell$ is $\sigma$-admissible where $\sigma$ can be taken as 
\begin{align}
\sigma  = \mathop {\sup }\limits_{y' \in {\cal Y}} \mathop {\sup }\limits_{\left| y \right| \le B\left( y \right)} \left| {\frac{{\partial c}}{{\partial y}}(y,y')} \right|. \nonumber
\end{align} 
\end{lemma}
\begin{proof}
The proof can be found in~\citep{b30}, on page 515, referring to the proof of Lemma 23. Note that this lemma differs slightly from the original one, but the proof is identical.

In the original lemma, the authors derived the following bound for the loss function: $$\forall z \in {\cal Z},0 \le \ell (A_s,Z) \le B\left( {\kappa\sqrt {\frac{{B(0)}}{\lambda }} } \right).$$ Here, $A_s$ denotes the optimal hypothesis returned by the SRM algorithm. Consequently, the upper bound is established specifically. 

In our work, the hypothesis $h$ learned by SRM is not necessarily the minimizer, and the assumption $0 \le c(h(X), Y) \le M$ is inherited from Theorem~\ref{theo2}. Therefore, we do not strictly follow the original lemmas; instead, we adopt only the parts that are relevant to our analysis. 

Specifically, the above result follows directly from the consideration of $\sigma$-admissibility. 
\end{proof}

\bibliography{sample}

\end{document}